\documentclass[letterpaper, 10 pt, conference]{ieeeconf}  

\IEEEoverridecommandlockouts                              

\usepackage{graphicx} 
\usepackage{amsmath} 
\usepackage{amssymb}  
\usepackage{amsbsy}

\usepackage{todonotes}
\usepackage{verbatim}

\usepackage{algorithm}
\usepackage{algpseudocode}

\usepackage[utf8]{inputenc}

\usepackage{cite}
\usepackage{subcaption}
\usepackage{booktabs}
\usepackage{tikz}
\usepackage{pgfplots}

\usepackage{url}

\usepackage{breakurl}
\usepackage[breaklinks]{hyperref}

\pgfplotsset{grid style={dotted, gray}}
\pgfplotsset{minor grid style={dotted,gray}}
\pgfplotsset{every tick label/.append style={font=\tiny}}
\pgfplotsset{every axis/.append style={font=\small}}
\pgfplotsset{ylabel near ticks}
\pgfplotsset{xlabel near ticks}
\newlength\figureheight 
\newlength\figurewidth

\title{\LARGE \bf
Learning from Demonstration for Hydraulic Manipulators
}

\author{Markku Suomalainen\textsuperscript{1*}, Janne Koivumäki\textsuperscript{2*}, Santeri Lampinen\textsuperscript{2}, Ville Kyrki\textsuperscript{1} and Jouni Mattila\textsuperscript{2}
\thanks{This work was supported by Academy of Finland under the project ``Cooperative heavy-duty hydraulic manipulators for sustainable subsea infrastructure installation and dismantling,'' under Grants 286553 and 286580.}
\thanks{\textsuperscript{1}M.\ Suomalainen and V.\ Kyrki are with School of Electrical Engineering, Aalto University, Finland. {\tt\small name.surname@aalto.fi} \newline
\hspace*{1.1em}\textsuperscript{2}J.\ Koivumäki, S. Lampinen and J.\ Mattila are with Tampere University of Technology, Finland. {\tt\small name.surname@tut.fi} \newline
\hspace*{0.7em} \textsuperscript{*}These authors contributed equally to this paper.
}}

\begin{document}

\maketitle
\thispagestyle{empty}
\pagestyle{empty}

\hyphenation{tele-op-er-a-tion}

\begin{abstract}
This paper presents, for the first time, a method for learning in-contact tasks from a teleoperated demonstration with a hydraulic manipulator. Due to the use of extremely powerful hydraulic manipulator, a force-reflected bilateral teleoperation is the most reasonable method of giving a human demonstration. An advanced subsystem-dynamic-based control design framework, virtual decomposition control (VDC),~is~used to design a stability-guaranteed controller for~the teleoperation system, while taking into account the full nonlinear dynamics of the master and slave manipulators. The use~of~fragile~force/ torque sensor at the tip of the hydraulic slave manipulator is avoided by estimating the contact forces from the manipulator actuators' chamber pressures. In the proposed learning method, it is observed that a surface-sliding tool has a friction-dependent range of directions (between the actual direction of motion and the contact force) from which the manipulator can apply force to produce the sliding motion. By this intuition, an intersection of these ranges can be taken over a motion to robustly find~a desired direction for the motion from one or more demonstrations. The compliant axes required to reproduce the motion can be found by assuming that all motions outside the desired direction is caused by the environment, signalling the need for compliance. Finally, the learning method is incorporated to a novel VDC-based impedance control method to learn compliant behaviour from teleoperated human demonstrations. Experiments with 2-DOF hydraulic manipulator with a 475kg payload demonstrate the suitability and effectiveness of the proposed method to perform learning from demonstration (LfD) with heavy-duty hydraulic manipulators.
\end{abstract}


\section{INTRODUCTION}
\label{intro}

Hydraulic actuation can provide many advantages over their electrical counterparts due to simplicity, robustness, low cost, and large power-to-weight ratio. For these reasons, hydraulically actuated heavy-duty work machines have been used for decades in various harsh-environment and risk-intensive industries such as agriculture, construction, forestry, and mining industries. In the recent years, hydraulic maintenance robots have also been extensively developed to operate heavy objects in harsh environments, such as in the nuclear fusion industry (see \cite{WHMAN} and \cite{Saarinen2011}) containing high radiation, extreme temperatures, and strong magnetic fields. It is already projected that the advent of robotics will revolutionize the (hydraulic) heavy-duty machine industry~\cite{Mattila_TMECH2017}. It is evident that in extreme working conditions where a stable high-bandwidth communication is not guaranteed, new solutions for automated operations are required to extend human capabilities to operate harsh environments safely.

A major challenge in automating heavy-duty hydraulic machinery is that performed tasks often require contact with the environment. As hydraulic manipulators are extremely powerful, the use of position control for in-contact tasks is not advised since a small error in the position would result in large forces applied against the environment, possibly leading to significant damage. This can be prevented by using compliance, i.e., allowing the manipulator to deviate from the planned trajectory in case of physical constraints. Impedance control \cite{Hogan1985}, in which a virtual spring-damper system is programmed to the controller, is a natural method for realizing compliant motions. However, automatic planning of compliant motions is shown to be mathematically infeasible \cite{canny1987new}. Furthermore, designing a stability-guaranteed (i.e., theoretically sound) and high-precision closed-loop control for multiple degrees-of-freedom (\textit{n}-DOF) hydraulic manipulators is a well-known challenge due to their highly nonlinear dynamic behaviour \cite{Mattila_TMECH2017,Sirouspour2001,Koivumaki_TRO2015,Koivumaki_TMECH2017}. It is valid to mention that the control system stability is the primary requirement for all control systems and an unstable system is typically useless and potentially dangerous \cite{Backstepping1995,Slotine1991}. As reviewed in \cite{Mattila_TMECH2017}, nonlinear model-based (NMB) control methods have shown to provide the most advanced control performance for hydraulic manipulators. However, it was only very recently in \cite{Koivumaki_TRO2015,Koivumaki_TMECH2017} where the authors managed to provide for the first time stability-guaranteed NMB control designs for \textit{n}-DOF hydraulic manipulators performing contact tasks.


Learning from Demonstration (LfD) is a well-established paradigm in robotics \cite{argall2009survey}. The idea is that a human teacher shows a demonstration by performing a specific task~and~the robot then learns to perform the same task even in slightly different environments. When dealing with heavy-duty hydraulic manipulators, teleoperation is the most logical choice to provide the demonstration, since this allows the domain experts to demonstrate the task, such as maintenance or earth moving, in a natural way. However, designing a bilateral force-reflected teleoperation for a hydraulic slave manipulator controlled with an electric master manipulator becomes a challenging control design task due to the slave's significant nonlinearities, its dynamic connection to the master, and~the need for arbitrary motion/force scaling between these significantly different manipulators. Moreover, demonstrations from teleoperation are shown to be on average noisier~than kinesthetic teaching \cite{fischer2016comparison}. Finally, learning compliant motions from human demonstrations requires the slave manipulator contact force measurement. Using a conventional force/torque sensor to measure the contact force is infeasible with hydraulic manipulators as these sensors are usually sensitive to shocks and overloading, situations that frequently occur in hydraulic operations.

Common tools for encoding LfD motions have the drawback that the position and force trajectories are coupled, rendering them ill-equipped to function in changing and hazardous environments where errors in position may grow large. We propose using the method presented in \cite{suomalainen2017}, where a task is modelled as a sequence of linear impedance controllers. From position and contact force data recorded in one or more demonstrations by haptic feedback teleoperation, we deduce the direction and stiffness parameters of the controller for performing a single motion. Sequencing the motions to perform a full task is outside of the scope of this paper, but there are existing methods \cite{kroemer2014learning} and ongoing research \cite{hagos2018} on how to learn and perform a whole task, such as excavation or heavy-duty manipulation or assembly.


In this paper, we show that despite the significant challenges, it is possible to realize LfD capabilities~for~hydraulic manipulators using a force-reflected bilateral teleoperation that allows an arbitrary motion and force scaling between~the manipulators. Furthermore, the proposed method takes an advantage of a novel LfD method \cite{suomalainen2017} which is incorporated to a novel impedance control method \cite{Koivumaki_TMECH2017} to learn compliant behaviour from teleoperated human demonstrations. The use of a fragile force sensor at the contact point is avoided by estimating the contact forces from the hydraulic cylinders' chamber pressures \cite{Koivumaki_TRO2015}. The experiments with a full-scale hydraulic slave manipulator (having an attached payload of 475 kg) demonstrate the feasibility of the proposed method. 







\vspace{-0.0cm}
\section{RELATED WORK}
\label{RELATED}
\vspace{-0.1cm}


The current state-of-the-art on control of hydraulic \textit{n}-DOF manipulators is reviewed in \cite{Mattila_TMECH2017}. This study shows~that stability-guaranteed NMB control methods can provide the most advanced control performance for highly nonlinear hydraulic manipulators. The study further shows that~the authors' stability-guaranteed NMB controls demonstrate the state-of-the-art for hydraulic manipulators' free-space motion (see \cite{Koivumaki_AutCon2015}) and constrained motion (see \cite{Koivumaki_TMECH2017,Koivumaki_TRO2015}) controls.

An extensive review on teleoperation in \cite{hokayem2006} shows that teleoperation~is~a~well-established paradigm with electric~systems. However, very few studies exist on teleoperation with hydraulic \textit{n}-DOF manipulators. In \cite{LampinenCASE2018, LampinenFPMC2018}, the authors proposed for the first time a full-dynamics-based force-reflected bilateral teleoperation for hydraulic \textit{n}-DOF manipulators. The previous studies on bilateral teleoperation with hydraulic \textit{n}-DOF manipulators, using linearized system models for master and slave manipulators, can be found in \cite{Salcudean1999,Tafazoli2002}. 

Common methods for encoding LfD skills are Dynamic Movement Primitive (DMP) \cite{schaal2006dynamic} and Hidden Semi-Markov Models \cite{racca2016} , where the trajectory is learned as a set of attractors. Recently, also force profile has been added to DMPs \cite{kramberger2017generalization}. However, the problem of DMPs with force profile is the tight positional coupling between force and position, which in turn means that there must not be a lot of variance between demonstrations or initial contact locations in case of chamfers. It has been shown that giving a peg-in-hole demonstration by teleoperation is more difficult than by kinesthetic teaching \cite{fischer2016comparison}. Therefore, the method for learning the skill in teleoperation must be robust against differences in demonstrations, such as different starting or ending positions. 

Recently, an LfD method for teleoperation was presented by Pervez et al. \cite{pervez2017novel}, who acknowledge the need for robustness when learning from teleoperated demonstrations. They choose one of the demonstrations as a reference trajectory and then extract the applicable data from other incomplete~or jerky demonstrations. In our method \cite{suomalainen2017}, this step is not necessary and demonstrations of different lengths and starting positions are naturally combined. Moreover, Pervez et al. used DMPs which, without the trajectory and force profile coupling, cannot take advantage of chamfers such as in Fig. \ref{fig:assembly} where contact force can appear from various directions, a common case in maintenance tasks. Finally, \textit{to the authors' best knowledge, LfD has previously not been shown to work with hydraulic manipulators for in-contact tasks.}

\begin{figure}[t]
\vspace{0.2cm}
\centering
\includegraphics[width=0.60\columnwidth]{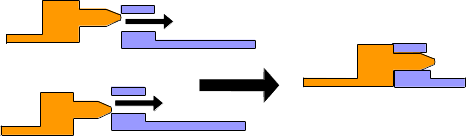}
\vspace{-0.1cm}
\caption{\small{Position alignment by taking advantage of chamfers.}}
\label{fig:assembly}
\vspace{-0.2cm}
\end{figure}

\vspace{-0.1cm}
\section{METHOD}
\label{METHOD}
\vspace{-0.1cm}

The goal is to show that our learning~method in \cite{suomalainen2017}~1)~can be used to learn impedance parameters from a teleoperated human demonstration with a hydraulic heavy-duty manipulator, and 2) can be incorporated to our novel impedance control method in \cite{Koivumaki_TMECH2017} such that the demonstrated motion can be reproduced with guaranteed stability, even in contact with varying environment and presence of measurement errors. Overall, the proposed method is built~on the authors' earlier contributions on high-precision and stability-guaranteed \textit{n}-DOF manipulator controls \cite{Koivumaki_TMECH2017,Koivumaki_AutCon2015,Koivumaki_TRO2015} and on LfD algorithms \cite{suomalainen2016,suomalainen2017}. Furthermore, the solution relies on an advanced force-reflected teleoperation method in \cite{zhu2000}, which is compatible to the hydraulic manipulator control methods in \cite{Koivumaki_TMECH2017,Koivumaki_AutCon2015,Koivumaki_TRO2015}. The proposed method consists of four important parts, i.e, \textit{control system design}, \textit{contact force estimation},  \textit{bilateral teleoperation} and \textit{learning algorithm}.

\vspace{-0.1cm}
\subsection{Control Method}
\label{sec:control}

As discussed, NMB control methods have shown to provide the most advanced control performance for hydraulic manipulators \cite{Mattila_TMECH2017}. The idea in these methods is to design a specific feedforward term (from the system inverse dynamics) to proactively generate the required actuator forces from the required motion dynamics. The feedforward term is designed to be responsible for the major control actions, whereas the feedback terms are used only to overcome~uncertainties, to maintain stability, and to address transition issues. 

Next, Section \ref{sec:VDC_control} describes the basis of the VDC~approach, which is used as an underlying NMB control~design framework for the hydraulic slave manipulator and electric master manipulator. Then, Section \ref{sec:Impedance_control} introduces the impedance control method \cite{Koivumaki_TMECH2017} suitable for the VDC framework. The method to connected the separately controlled master and slave plants for the force-reflected bilateral teleoperation is discussed in Section \ref{sec:teleop}. \textit{The detailed control designs for the hydraulic slave manipulator and electric master manipulator can be found in \cite{Koivumaki_TMECH2017} and \cite{LampinenCASE2018, LampinenFPMC2018}.}

\subsubsection{Virtual Decomposition Control}
\label{sec:VDC_control}

VDC \cite{Zhu_VDC,Zhu1997} is a novel NMB subsystem-dynamics-based control design method, developed for controlling complex robotic systems. The method allows the original system (see Fig. \ref{fig:studied_system}a) to be \textit{virtually decomposed} to modular \textit{subsystems}, objects and open chains (see Fig. \ref{fig:studied_system}b), using conceptual \textit{virtual cutting points} (VCPs). This enables that the control system design and its stability analysis can be performed locally at the subsystem level without imposing additional approximations.

After the virtual decomposition, the system is represented by a simple oriented graph (SOG); see Fig.~\ref{fig:studied_system}c. In the SOG, each subsystem represents a \textit{node}, and each VCP represents a \textit{directed edge}, the direction of which defines the force reference direction. Then, in the control system design, the kinematics of subsystems can be computed by propagating along the direction of the VCP flow in the SOG from the \textit{source node} (object 0) toward the \textit{sink node} (object 2); see Fig. \ref{fig:studied_system}c. Using the kinematics, the dynamics of subsystems can be computed by propagating along the opposite direction of the SOG from the \textit{sink node} toward the \textit{source node}. Finally, the subsystems' control design can be established using the system kinematics and dynamics.

\begin{figure}[t]
			\vspace{0.20cm}
	  \centering
      \includegraphics[width=0.80\columnwidth]{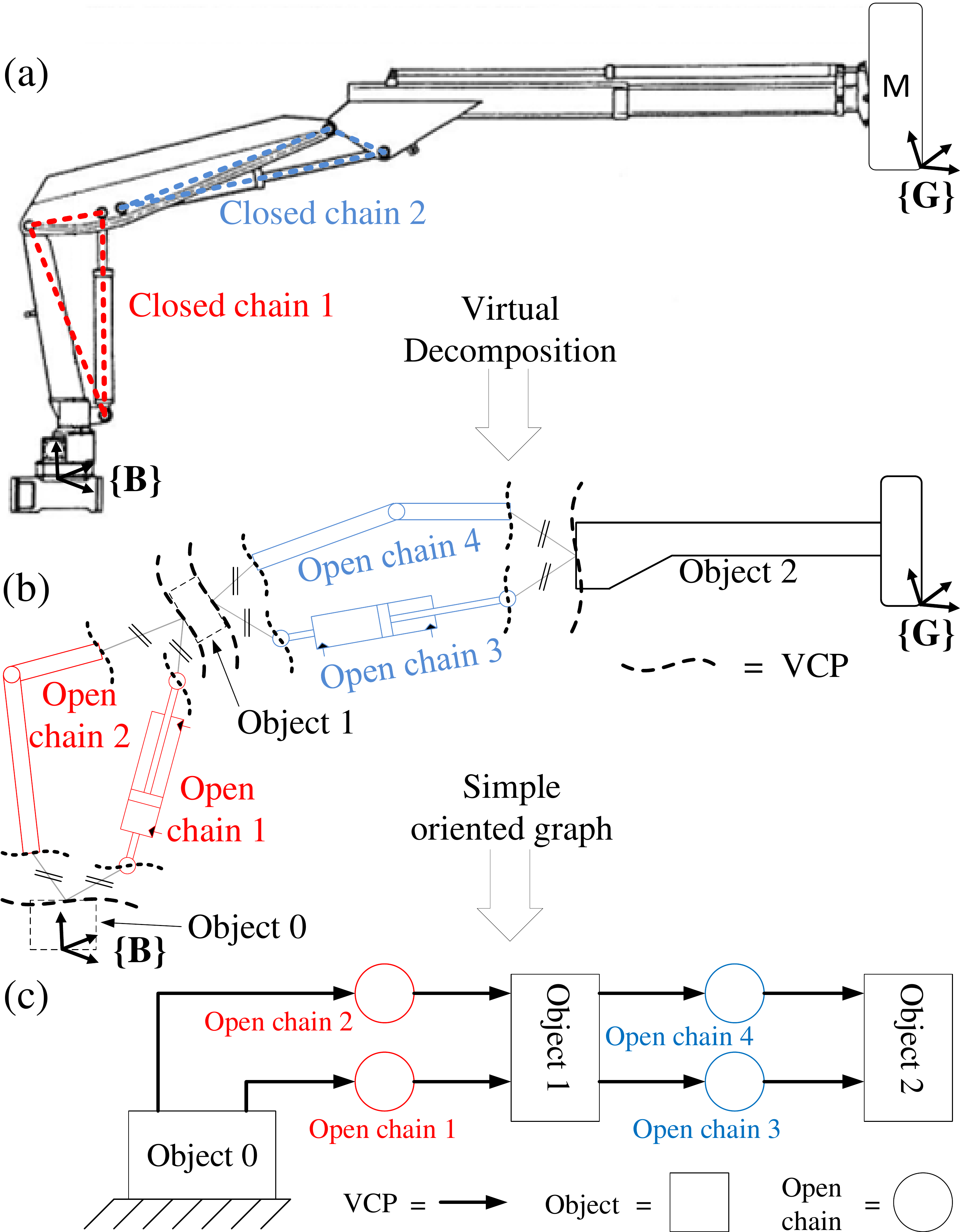}
      \vspace{-0.1cm}
      \caption{\small{(a) Two-DOF hydraulic manipulator. (b) A virtual decomposition of the system. (c) A simple oriented graph of the system.}
      \label{fig:studied_system}}
	  \vspace{-0.6cm}
\end{figure}

The unique control design philosophy of VDC has brought a \textit{modularity} to control system engineering, enabling, e.g., that changing the control (or dynamics) of one subsystem does not affect the control equations of the rest of the system. Furthermore, an adaptive control can be incorporated to the control design to cope with uncertain parameters involved in subsystem dynamics.
With advances of VDC, the state-of-the-art control performances are already demonstrated for hydraulic manipulators in their free-space motion control and in constrained motion control; see \cite{Zhu2005, Mattila_TMECH2017,Koivumaki_TMECH2017,Koivumaki_AutCon2015,Koivumaki_TRO2015}.

\subsubsection{The Proposed Impedance Control}
\label{sec:Impedance_control}

In view of \cite{Hogan1985}, the full expression of the Cartesian target impedance law for manipulators can be described as  
\begin{equation}
		{}^{\mathbf G}{\pmb{F}}_{\rm d} - {}^{\mathbf G}{\pmb{F}} = -\mathbf{M}_{\rm d}(\ddot{\pmb{\chi}}_{\rm d} - \ddot{\pmb{\chi}}) -\mathbf{D}_{\rm d}(\dot{\pmb{\chi}}_{\rm d} - \dot{\pmb{\chi}}) -\mathbf{K}_{\rm d}({\pmb{\chi}}_{\rm d} - {\pmb{\chi}}) \label{EQ_imp1}
\end{equation}
where  ${\pmb{\chi}} \in \mathbb{R}^{n}$ and ${}^{\mathbf G}{\pmb{F}} \in \mathbb{R}^n$ are the Cartesian position and contact force vectors; ${\pmb{\chi}}_{\rm d} \in \mathbb{R}^{n}$ and ${}^{\mathbf G}{\pmb{F}}_{\rm d} \in \mathbb{R}^n$ are the desired Cartesian position and contact force vectors; and $\mathbf{M}_{\rm d} \in \mathbb{R}^{n\times n}$, $\mathbf{D}_{\rm d} \in \mathbb{R}^{n\times n}$ and $\mathbf{K}_{\rm d} \in \mathbb{R}^{n \times n}$ characterize the desired inertia, damping and stiffness, respectively. Then, similar to \cite{Koivumaki_TMECH2017,Semini2015}, neglecting the inertia term in~\eqref{EQ_imp1}, the target impedance can be written as
\begin{equation}
		{}^{\mathbf G}{\pmb{F}}_{\rm d} - {}^{\mathbf G}{\pmb{F}} = -\mathbf{D}_{\rm d}(\dot{\pmb{\chi}}_{\rm d} - \dot{\pmb{\chi}}) -\mathbf{K}_{\rm d}({\pmb{\chi}}_{\rm d} - {\pmb{\chi}}). \label{EQ_imp2}
\end{equation}

The VDC approach is a \textit{velocity-based control method}, which takes care of the system dynamics \cite{Zhu_VDC}. In VDC, a required velocity serves as a reference trajectory for a system and the control objective is to make the controlled actual velocities track the required velocities. The general format of a required velocity includes a desired velocity (which usually serves as a reference trajectory for a system) and one or more terms that are related to control errors~\cite{Zhu_VDC}. The following control law (required velocity vector $\dot{\pmb{\chi}}_{\rm r}$) was proposed in \cite{Koivumaki_TMECH2017} to perform the impedance control within the VDC framework
\begin{equation} \label{EQ_chi_r}
	\dot{\pmb{\chi}}_{\rm r} = \dot{\pmb{\chi}}_{\rm d} + \pmb{\Lambda}_{\chi}(\pmb{\chi}_{\rm d} - \pmb{\chi}) + \pmb{\Lambda}_f({}^{\mathbf G}{\pmb{F}}_{\rm d} - {}^{\mathbf G}{\pmb{F}})
\end{equation}
where  $\pmb{\Lambda}_{\chi} \in \mathbb{R}^{n\times n}$ and $\pmb{\Lambda}_f \in \mathbb{R}^{n\times n}$ are two positive-definite matrices characterizing Cartesian position and force control and they should be defined according to Condition~\ref{cond:gain2}.

\newtheorem{cond}{Condition}     
\begin{cond} \label{cond:gain2}
Matrices $\pmb{\Lambda}_{f}$ and $\pmb{\Lambda}_{\chi}$ should be defined as
	\begin{align}
		\pmb{\Lambda}_{f} = \mathbf{D}^{-1}_{\rm d},\hspace{0.3cm} \pmb{\Lambda}_{\chi} = \mathbf{D}^{-1}_{\rm d}\mathbf{K}_{\rm d} \nonumber
	\end{align}
such that both $\pmb{\Lambda}_{f}$ and $\pmb{\Lambda}_{\chi}$ qualify as positive-definite.
\end{cond}
Then, the following Theorem~\ref{thm:impedance} provides that the target impedance behaviour~\eqref{EQ_imp2} can be achieved.
\newtheorem{theorem}{Theorem}
\begin{theorem} \label{thm:impedance}
Consider the proposed control law \eqref{EQ_chi_r}, which defines the required velocity behaviour for the system. If $\pmb{\Lambda}_{f}$ and $\pmb{\Lambda}_{\chi}$ in \eqref{EQ_chi_r} are defined according to Condition~\ref{cond:gain2}, then the control law \eqref{EQ_chi_r} equals the target impedance law \eqref{EQ_imp2}.            
\end{theorem}
\begin{proof}
See Appendix~\ref{proof:impedance}.
\end{proof}


\subsection{Contact Force Estimation}
\label{sec:force_estimation}

A contact force control requires a force feedback. However, the use of conventional six-DOF force/moment sensor (built using either straingauge technology or optics) at a manipulator tip is not practical with extremely powerful hydraulic manipulators. Thus, alternative methods for contact force measuring (estimation) are highly favorable.   

A force-sensorless contact force control method was developed in \cite{Koivumaki_TRO2015} by the authors. In the method, the contact forces were estimated from chamber pressures of the manipulator's hydraulic actuators by using an accurate system modeling and gravity compensation; see \cite{Koivumaki_TRO2015}. However, estimating contact forces from the chamber pressures is challenging due to the nonlinear dynamic behaviour of hydraulic manipulators. Furthermore, in the estimation method in \cite{Koivumaki_TRO2015}, the system inertia and piston friction were not considered in the contact force estimation. Thus, the greater the manipulator velocity, the more error in the contact force estimates. To improve the contact force estimation accuracy, e.g., the method presented in \cite{Vihonen2016} can be used to address the system inertia through accurate estimation of manipulators' link accelerations.

Eventually, inaccuracies will always exist in the contact force estimation with hydraulic manipulators. This~is~due~to the systems complexity involved with non-smooth and discontinuous nonlinearities, and model and parameter uncertainties. Thus, in this study, \textit{a major emphasis is not paid~to the contact force estimation accuracy, as the learning method should be robust for the force estimation errors}.

\subsection{Teleoperation}
\label{sec:teleop}

Teleoperation can extend human capabilities to hard-to-reach, hazardous or dangerous environments. In addition,~it can be used to scale control actions to both micro and macro environments, e.g., to surgery or control of an excavator. Furthermore, when dealing with heavy and powerful hydraulic manipulators, kinesthetic teaching is impossible to perform, and teleoperation is the remaining option. With force-reflected bilateral teleoperation, the operator is able to \textbf{1}) send (control) the desired motions to a slave manipulator with a master manipulator and \textbf{2}) physically feel forces acting at the slave manipulator.

The designed  force-reflected bilateral teleoperation system is identical with the system presented in \cite{LampinenCASE2018, LampinenFPMC2018}. Individual NMB controllers are designed for both master and slave manipulator according to VDC design principles defined in section \ref{sec:VDC_control}. Then, the two separately controlled plants~are connected with each other using the following equations~\cite{zhu2000}
%
\begin{align}
\mathbf{v}_{sr} &= \kappa_p\mathbf{\tilde{v}}_m
+ \mathbf{\Lambda}\big(\,\kappa_p\mathbf{\tilde{p}}_m - \mathbf{p}_s\: \big)
- \; \mathbf{A} \, \big(\,\mathbf{\tilde{f}}_s + {\kappa}_f
\mathbf{\tilde{f}}_m\big)
\label{eq:teleoperationVsd}\\
\mathbf{v}_{mr} &= \frac{1}{\kappa_p}\mathbf{\tilde{v}}_s \,
+ \mathbf{\Lambda}\big(\frac{1}{\kappa_p}\mathbf{\tilde{p}}_s \,- \mathbf{p}_m\big) 
- \frac{\mathbf{A}}{\kappa_p}\big(\,\mathbf{\tilde{f}}_s + {\kappa}_f
\mathbf{\tilde{f}}_m\big)\hspace{-0.5cm}
\label{eq:teleoperationVmd}
\end{align}
where $ \mathbf{\Lambda} \in \mathbb{R}^{n \times n}$ is a diagonal positive-definite matrix defining position feedback gain, $ \mathbf{A} \in \mathbb{R}^{n \times n} $ is a positive-definite matrix defining force feedback gain, and $  \mathbf{p}_m $ and  $ \mathbf{p}_s $ denote the position/orientation of the master and slave manipulator, respectively, subject to $  \dot{\mathbf{p}}_m = \mathbf{v}_m $ and  $ \dot{\mathbf{p}}_s = \mathbf{v}_s $. In \eqref{eq:teleoperationVsd} and \eqref{eq:teleoperationVmd}, tilde $\mathbf{\tilde{\cdot}}$ denotes that the variable is obtained with a first order low-pass filter in \cite{zhu2000}, $ \mathbf{\tilde{f}}_m \mbox{ and } \mathbf{\tilde{f}}_s $ are the filtered contact forces from the master/slave manipulator toward the operator/environment. Scaling factors ${\kappa_p}$~and~${\kappa_f}$ are used to scale the position and force of the master manipulator, respectively. \textit{The scaling factors allow arbitrary motion and force scaling between the manipulators \cite{Zhu_VDC}}.


The asymptotic position/velocity tracking between the master and slave manipulators can be guaranteed in free-space and constrained motions by following the control design principles in \cite{Zhu_VDC,zhu2000}. Finally, in \cite{Salcudean1999,Tafazoli2002}, linearized system models were used in hydraulic manipulator's bilateral teleoperation, whereas \textit{our method takes the full nonlinear dynamics of the master and slave manipulators into account}.

\subsection{Learning}
\label{sec:learning}
\vspace{-0.2cm}

To reproduce a motion with linear dynamics, the parameters $\pmb{\chi}_{\rm d}$ and $\pmb{\Lambda}_{\chi}$ in \eqref{EQ_chi_r} must be deduced from the demonstration data. Since an NMB controller requires the trajectory to be differentiable, we compute $\pmb{\chi}_{\rm d}$ as a quintic path \cite{Jazar2010}, providing a smooth trajectories for position, velocity and acceleration. The trajectory is created from starting position $\pmb{\chi_1}$ and ending position $\pmb{\chi_n}$, where $\pmb{\chi_n}=\pmb{\chi_1}+\delta \pmb{\hat{v}}_{\rm d}^*$ with $\delta$ defining the length of the trajectory and $\pmb{\hat{v}}_{\rm d}^*$ the desired direction learned from demonstrations. $\pmb{\Lambda}_{\chi}$ we compute by Condition \ref{cond:gain2} from the traditional impedance control stiffness matrix $\mathbf{K}_{\rm d}$. Now, a controller with correctly learned parameters $\pmb{\hat{v}}_{\rm d}^*$ and $\mathbf{K}_{\rm d}$ can reproduce motions such as in Fig. \ref{fig:assembly} where the same controller can perform both depicted motions which result in the completed assembly. First, we present how to learn $\pmb{\hat{v}}_{\rm d}^*$ and, then, $\mathbf{K}_{\rm d}$ in the Cartesian space.

\subsubsection{Learning desired direction}
\label{sec:method_desireddir}
The intuition for learning the desired direction $\pmb{\hat{v}}_{\rm d}^*$ stems from geometry: to slide the robot's end-effector along a surface, there is always a friction-dependent sector $s$ of directions from which the robot can apply a force to accomplish the sliding. If this sector is calculated at intervals over a whole demonstration, the intersection of all sectors $s_i$ would signify a direction which can lead the end-effector through the whole demonstrated motion either in free space or in contact. We call sector $s$ a set of desired directions and it is visualized for a single time-instant in 2-D in Fig.  \ref{fig:sliding_forces}. From Fig. \ref{fig:sliding_forces} we also see that the force estimated in Section \ref{sec:force_estimation} consists of
 \begin{equation}
  {}^{\mathbf G}{\pmb{F}} = -\pmb{F_N} - \pmb{F_{\mu}}
  \label{eqt:measured_force}
\end{equation}
where $\pmb{F_{\mu}}=\vert\mu\pmb{F_N}\rvert\left( -{\pmb{\hat{v}_a}}\right)$ is the force caused by Coulomb friction with $\mu$ being the friction coefficient, $\pmb{\hat{v}_a}$ the actual direction of motion and $\pmb{F_N}$ the normal force. As explained in Section \ref{sec:force_estimation}, we ignore the acceleration from \eqref{eqt:measured_force} and build the algorithm robust enough to withstand the error. 
 Throughout this paper, we will use the circumflex (\^{}) notation to denote the normalization of a vector. From \eqref{eqt:measured_force} and Fig. \ref{fig:sliding_forces} we can see that $s$ is between the direction of contact force ${}^{\mathbf G}{\pmb{F}}$ and the actual direction of motion $\pmb{\hat{v}_a}$. 

\begin{figure}[t]
\vspace{0.2cm}
	\centering
	\includegraphics[width=.70\columnwidth]{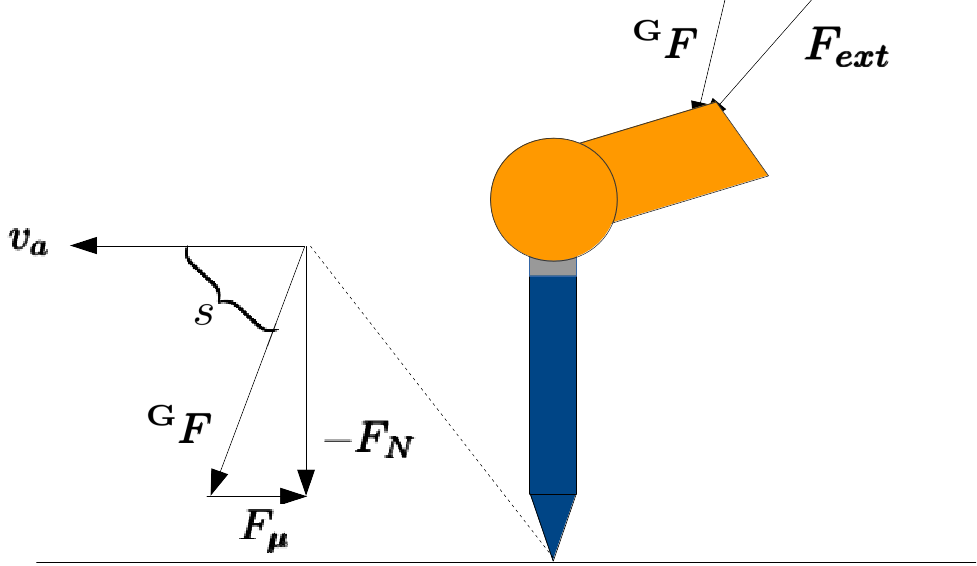}
	\vspace{-0.1cm}
	\caption{\small{Illustration of the forces acting on the end-effector during a sliding motion. ${}^{\mathbf G}{\pmb{F}}$ is the contact force estimated from chamber pressures, $\pmb{F}_N$ the normal force, $\pmb{F}_{\mu}$ the friction force, $\pmb{F}_{ext}$ the external force applied during a demonstration, $\pmb{v}_a$ the actual direction of motion and $s$ the sector of desired directions.}
	\label{fig:sliding_forces}}
\vspace{-0.4cm}
\end{figure} 

\begin{figure}[t]
\vspace{0.2cm}
\centering
\includegraphics[width=.38\columnwidth]{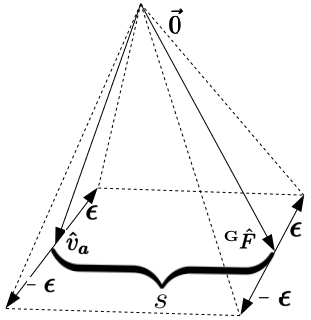}
\caption{\small{Illustration of expanding 2-D sector $s$ into 3-D set of directions $P$ in (\ref{eqt:epsilon}) and (\ref{eqt:polyhedron}). Continuous lines represent the vectors and dotted lines highlight the pyramid shape. }\label{fig:single_pyramid}}
\vspace{-0.0cm}
\end{figure}

As a human cannot give a perfect demonstration, sector $s$ must be extended perpendicularly to allow the calculation of an intersection from real demonstrations in 3-D. We compute the vector $\epsilon$ which extends the sector $s$ perpendicularly by $\alpha$ degrees
\begin{equation}
\pmb{\epsilon} = \tan \alpha \frac{{}^{\mathbf G}{\pmb{\hat{F}}} \times \pmb{\hat{v}_{a}} }{\lvert ^{\mathbf G}{\pmb{\hat{F}}} \times \pmb{\hat{v}_{a}} \rvert }.
\label{eqt:epsilon}
\end{equation}
Now, by adding $\epsilon$ and $-\epsilon$ to both $\pmb{\hat{F}}$ and $\pmb{\hat{v}_{a}}$, we can construct polyhedron $P$ defining the set of desired directions at each time instant. This is computed in (\ref{eqt:polyhedron}) and visualized in Fig.~\ref{fig:single_pyramid}. 
\begin{equation}
P_k=
\begin{bmatrix}
\pmb{\hat{v}_{a}}+\pmb{\epsilon} &
\pmb{\hat{v}_{a}}-\pmb{\epsilon} &
^{\mathbf G}{\pmb{\hat{F}}}-\pmb{\epsilon} &
^{\mathbf G}{\pmb{\hat{F}}}+\pmb{\epsilon}
\end{bmatrix}^T.
\label{eqt:polyhedron}
\end{equation}
Essentially, $\pmb{\hat{v}}_{\rm d}^*$ for the whole motion is found by projecting each polyhedron $P_k$ into a 2-D polygon, calculating the intersection of the polygons with outlier rejection and, finally, projecting a chosen point from the intersection back into a 3-D vector, which is the final chosen $\pmb{\hat{v}}_{\rm d}^*$. More details about this process, such as the exact projection algorithms, can be found from \cite{suomalainen2017}. If there are more than one demonstrations, the force and position data are concatenated and the intersection is calculated over all the demonstrations, naturally combining multiple demonstrations.

\subsubsection{Finding the compliant axes}
\label{sec:method_compliant}

\begin{figure}[t]
\centering
\begin{subfigure}[b]{0.48\columnwidth} 
	\centering
	\includegraphics[width=0.74\columnwidth]{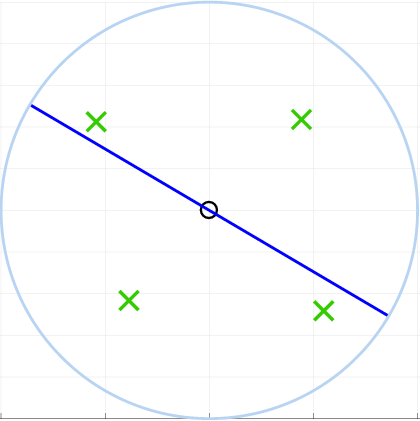}
	\caption{2 compliant axes}
	\label{fig:2_compl}
\end{subfigure}
\begin{subfigure}[b]{0.48\columnwidth} 
	\centering
	\includegraphics[width=0.74\columnwidth]{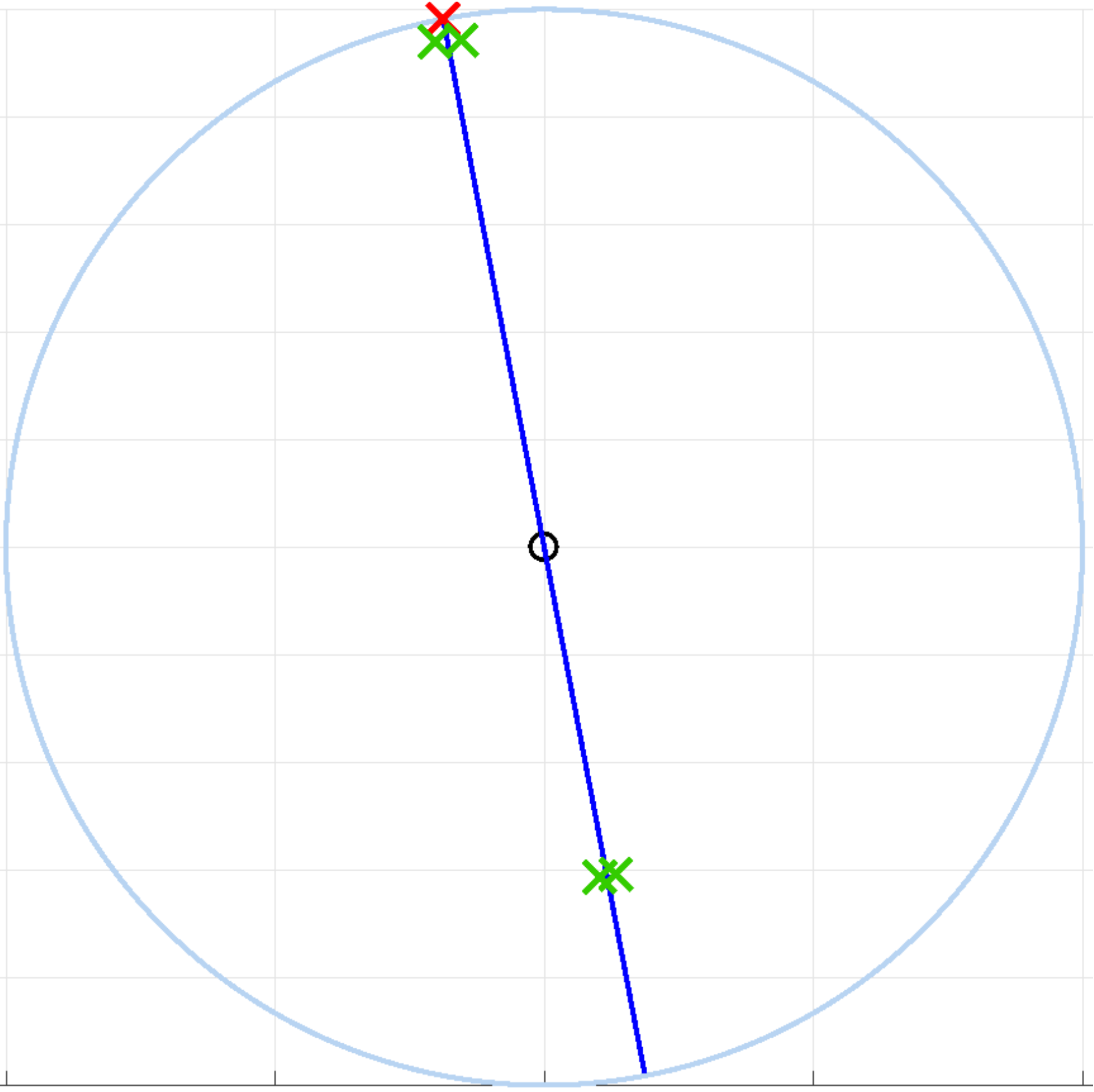}
	\caption{1 compliant axis}
	\label{fig:1_compl}
\end{subfigure}
\caption{\small{Unit spheres in the coordinate system for determining the compliant axes. Origin (black circle) represents $\pmb{\hat{v}}_{\rm d}^*$,~green~crosses are the corresponding $\pmb{\hat{v}_{a}}$ of each demonstration, blue line visualizes the model for 1 compliant axis and red cross the identified compliant direction when the model with 1 compliant axis is chosen.}}
\label{fig:comp_axes}
\vspace{-0.2cm}
\end{figure}

After finding the desired direction, we need to find the compliant axes to construct $\mathbf{K}_{\rm d}$ such that the task can be reproduced. Our key idea is to assume that if motion is observed in other directions besides the desired direction $\pmb{\hat{v}}_{\rm d}^*$, this motion is caused by the environment and therefore compliance is required. We assume that if compliance is required in a direction, the stiffness in that direction should be zero, i.e., the corresponding axis is fully compliant. This leads to the requirement that the compliant axes need to be perpendicular to $\pmb{\hat{v}}_{\rm d}^*$, since no motion can be commanded along a direction where stiffness is zero. This requirement restricts our search into directions perpendicular to $\pmb{\hat{v}}_{\rm d}^*$, which when projected into 2-D and normalized form a unit sphere visualized in Fig. \ref{fig:comp_axes}. From this unit sphere we find the number of compliant axes. Zero compliant axes are detected when there were no deviations from the desired direction, i.e., all the observed motion $\pmb{\hat{v}_{a}}$ was along $\pmb{\hat{v}}_{\rm d}^*$, signalling free space motion. In such a case, the green crosses in Fig. \ref{fig:comp_axes} would be on top of the black circle. When $\pmb{\hat{v}_{a}}$ deviates from $\pmb{\hat{v}}_{\rm d}^*$ along a single line in this coordinate system, there has been deviation from $\pmb{\hat{v}}_{\rm d}^*$ along a single axis and, therefore, one compliant axis is required, as in Fig. \ref{fig:1_compl}. The direction of this axis can be obtained by projecting the intersection of the axis and unit sphere back to 3-D. If the deviation from $\pmb{\hat{v}}_{\rm d}^*$ is not along a single line, both axes perpendicular to $\pmb{\hat{v}}_{\rm d}^*$ must be compliant, as in Fig. \ref{fig:2_compl}. The whole process is shown in Algorithm \ref{alg:compliant}. Again, we refer the reader to \cite{suomalainen2017} for a more detailed explanation.

\begin{algorithm}[b]
\caption{Finding the required number of compliant axes and their directions.}
\label{alg:compliant}
\begin{algorithmic}[1]
\small{
\State {Set of mean actual directions of each $i$ demonstration $\pmb{\overline{v}_a}$}
\State{Rotate and project $\pmb{\overline{v}_a}$ to 2-D such that the origin represents $\pmb{\hat{v}}_{\rm d}^*$}
\For{$d=0:2$ axes of compliance}
\For{$\pmb{v_i}$ in $\pmb{\overline{v}_a}$}
\State{$\pmb{\epsilon_{di}} = \left( I-U_d \right)  \pmb{v_i}$}
\State{where $U_d=$ rank $d$ PCA approximation of $\pmb{\overline{v}_a}$}
\EndFor
\State{$L_d = \prod \limits_{i} \mathcal{N}\left(\pmb{\epsilon_{di}}\vert \pmb{0},\Sigma\right)$}
\State{Calculate $BIC_d$ with (\ref{eqt:bic})}
\EndFor
\State{$D = \arg \min_d BIC_d$} 
\If{$D=1$}
\State{Project $U_d$ back to 3-D} 
\EndIf
}
\end{algorithmic}
\end{algorithm}

To compute numerical values for different number of compliant axes, we calculate the likelihoods for each case on rows 3-8 in Algorithm \ref{alg:compliant}. This is done by utilizing Principal Component Analysis (PCA) approximations of different ranks. Rank 0 approximation corresponds to 0 matrix, and therefore the likelihood is based on the distance from origin. Rank 1 approximation corresponds to a 1st degree polynomial fit in $\pmb{\overline{v}_a}$, and rank 2 corresponds to maximum likelihood since it perfectly explains the data. In this case, the direction of compliance is not required from $U_2$, since the whole plane perpendicular to $\pmb{\hat{v}}_{\rm d}^*$ must be compliant.

Finally, we wish to encourage the use of simpler models. We take inspiration from Bayesian Information Criterion (BIC) \cite{schwarz1978estimating}, which is defined
\begin{equation}
BIC = \ln(n)k-2\ln(L)
\label{eqt:bic}
\end{equation}
where $n$ is the number of data points, $k$ the number of parameters and $L$ the likelihood of a model. We choose the model with the lowest BIC value as the number of required compliant axes on row 8 of Algorithm \ref{alg:compliant}, and in case of 1 compliant axis also the direction. The details of writing the stiffness matrix $\mathbf{K}_{\rm d}$ can be found in \cite{suomalainen2016}. We note that this is not the typical use of BIC where $n\gg k$ and variance in the likelihood is calculated from the data, but we assume that the uncertainty of demonstrations can be estimated beforehand.

\section{EXPERIMENTS AND RESULTS}
\label{EXPERIMENTS}

\begin{figure}[b]
			\vspace{-0.2cm}
      \centering
      \makebox{\parbox{\columnwidth}{\centering \includegraphics[width=\columnwidth]{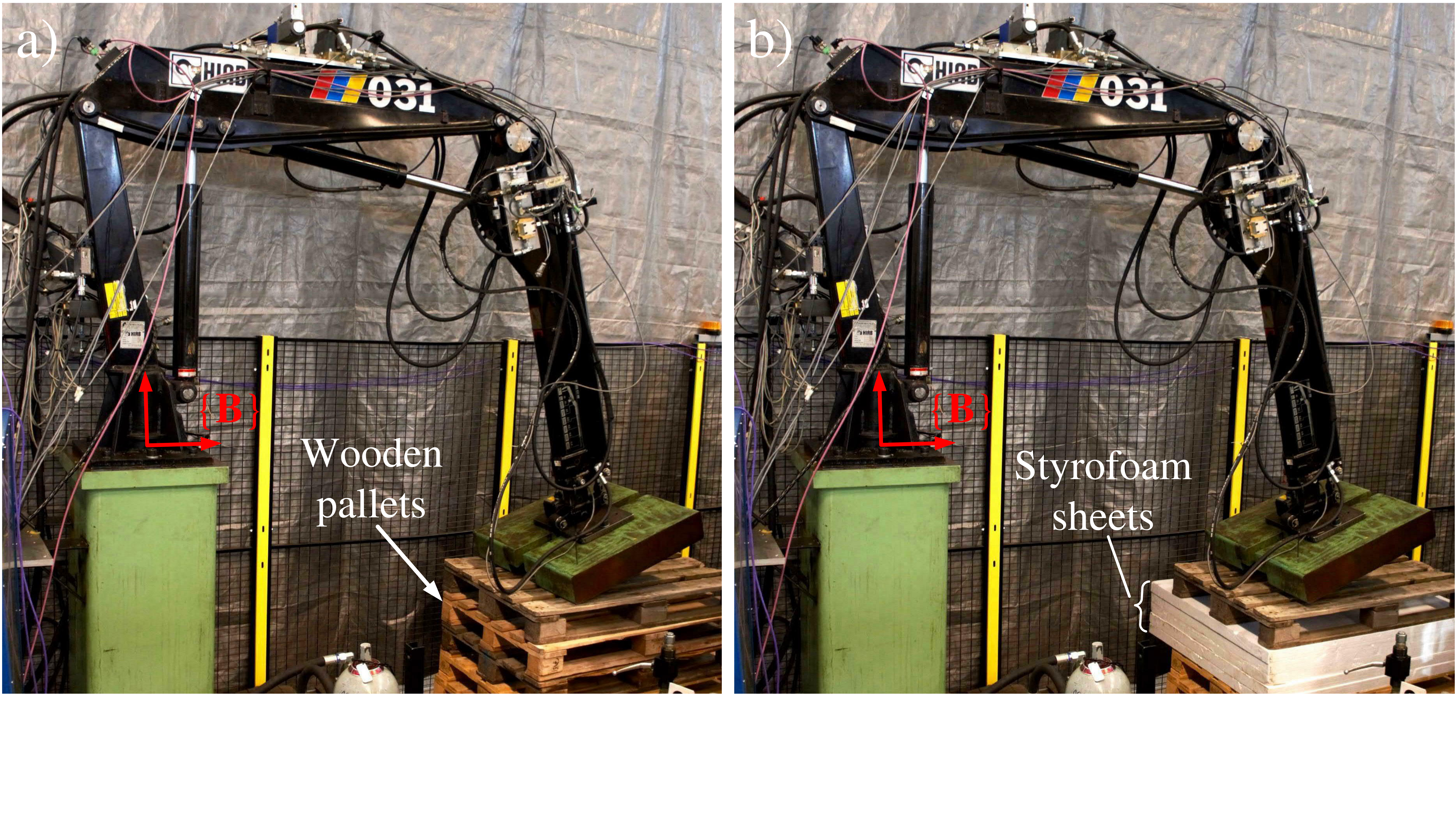}
      \vspace{-0.6cm}
			\caption{\small{a) Experiment setup with wooden pallets. b) Experiment setup with styrofoam sheets and wooden pallets. The manipulator’s position in the figures show the starting point of the test trajectories (same starting position in the both cases).}}
      \label{fig:setup}}}
\end{figure}

This section demonstrates the suitability and effectiveness of the proposed method to perform LfD with hydraulic~manipulators. First, Section \ref{sec:results_learning} demonstrates the learning method by recording data from the motion of sliding along a stack of wooden pallets (see Fig. \ref{fig:setup}a), from which we gathered four demonstrations via the teleoperation method described in Section \ref{sec:teleop}. Then, Section \ref{sec:results_reprod} shows that with the parameters learned from the demonstrations we can reproduce the motion in the wooden pallet environment (in Fig. \ref{fig:setup}a) but also in another environment, where the styrofoam sheets were used (see Fig. \ref{fig:setup}b) to change the properties of the original environment. \textit{In the below sections, the Cartesian position/force data refers the data in relation to the system base frame} \{\textbf{B}\}; see Fig. \ref{fig:setup}.

The two-DOF hydraulic manipulator (in~Fig.~\ref{fig:setup}) has the maximum reach of approximately 3.2 m and the payload of 475 kg is attached to its tip. Phantom Premium 3.0 haptic device is used as the electric master manipulator in the realized bilateral teleoperation system. For the real-time control system implementation, the following components were used: DS1005 processor board, DS3001 incremental encoder board, DS2103 DAC board, DS2003 ADC board, and DS4504 100 Mb/s ethernet interface. The remaining of the hardware implementations can be found in \cite{Koivumaki_TRO2015} or \cite{Koivumaki_TMECH2017}.


\subsection{Learning}
\label{sec:results_learning}

\begin{figure}[t]
      \centering
      \makebox{\parbox{\columnwidth}{\centering \includegraphics[width=\columnwidth]{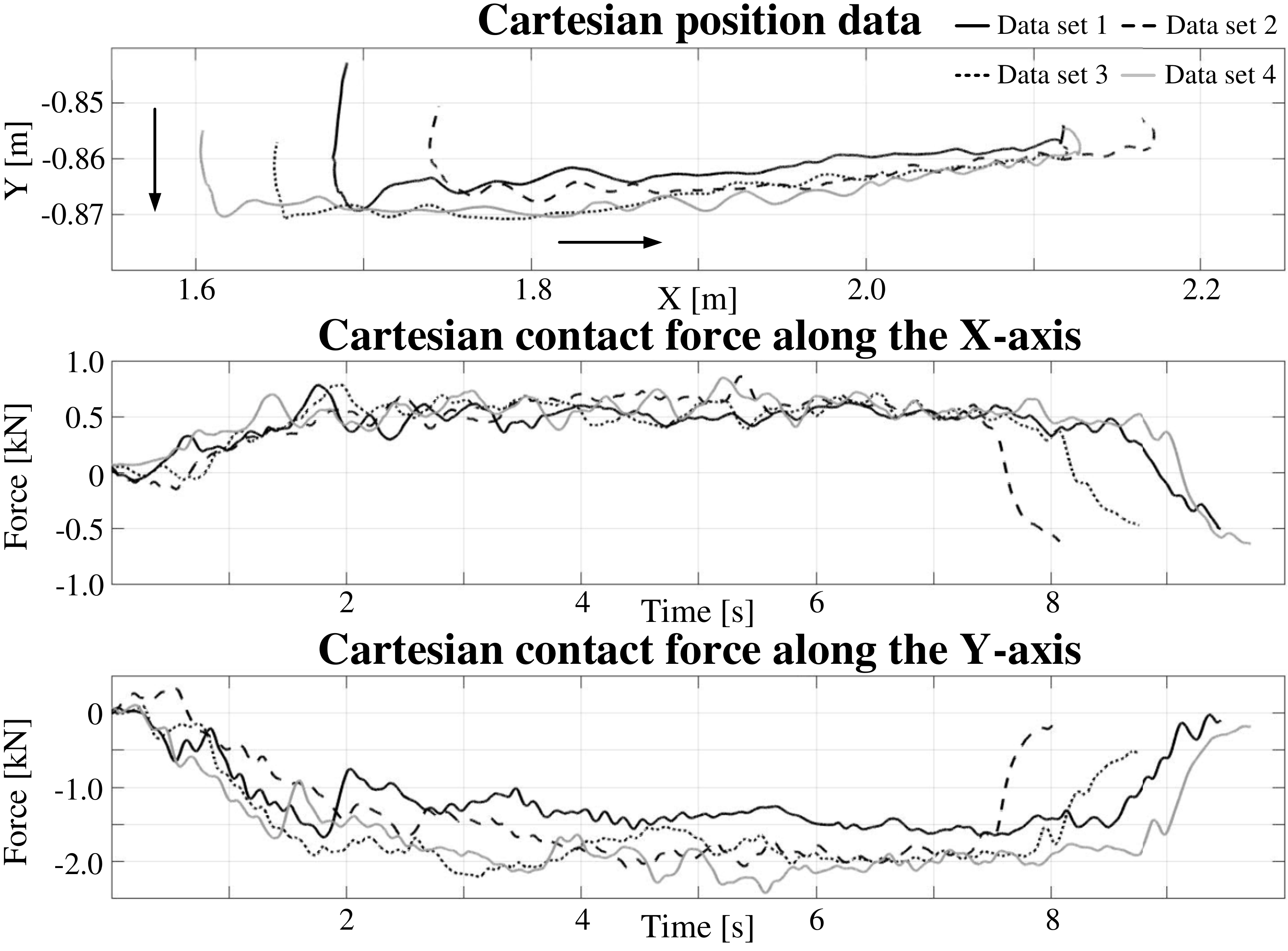}
      \vspace{-0.6cm}
			\caption{\small{The learning data. The learning data was gathered by using the wooden pallets in Fig. \ref{fig:setup}a as the contact environment.}}
      \label{fig:learned_data}}}
			\vspace{-0.2cm}
\end{figure}

\begin{figure}[b]
			\vspace{-0.4cm}
      \centering
      \makebox{\parbox{\columnwidth}{\centering \includegraphics[width=0.55\columnwidth]{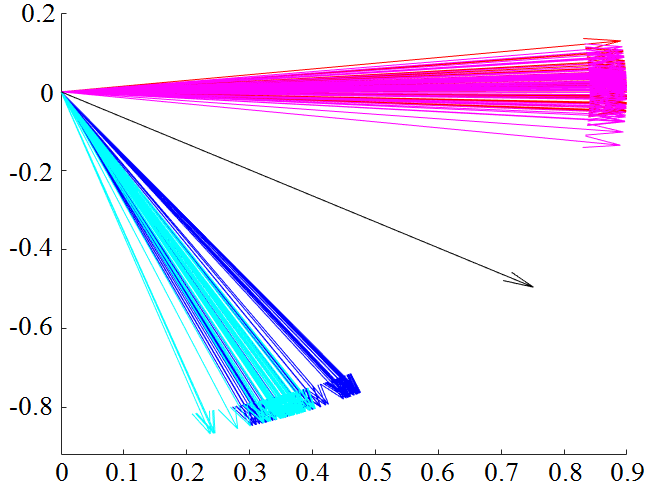}
      \vspace{-0.2cm}
			\caption{\small{2-D plot of the limits for desired directions from demonstrations 1 and 2. Directions of motion (demo 1 red, demo 2 magenta), interaction forces (blue, cyan) and the desired direction (black). }
      \label{fig:desired}}}}
\end{figure}

Our goal is to validate that the desired direction $\pmb{\hat{v}}_{\rm d}^*$, computed with the algorithm described in Section \ref{sec:method_desireddir}, can reproduce the demonstrated motion. Since the slave manipulator of the experiments is two-DOF, $\pmb{\hat{v}}_{\rm d}^*$ is a 2-D vector. Thus, we can assume there is no deviation perpendicular to $s$, and adding $\epsilon$ in (\ref{eqt:polyhedron}) is not required. Otherwise, the computation is performed similarly as it would be for a 3-D vector (with 0 values for z-axis).

The position and force data from the four demonstrations of sliding the manipulator tip along a wooden pallet (see Fig. \ref{fig:setup}a) are plotted in Fig. \ref{fig:learned_data}. The original measurement frequency was 500Hz, which we averaged to 25Hz. Fig. \ref{fig:desired} visualizes in 2-D the directions of motion and the directions of estimated forces in 25Hz (i.e. the edges of $P$ from Fig. \ref{fig:single_pyramid} in 2-D without $\epsilon$) from demonstrations 1 and 2, starting the moment when contact force was detected. As expected, in this simple setup the desired direction is detected approximately in the middle of these limitations, corresponding to the center of intersection between sectors $s$ from Fig. \ref{fig:sliding_forces}. $\pmb{\hat{v}}_{\rm d}^*$ computed from demonstrations 3 and 4 produced a very similar result. The expected results are achieved despite the difficulties in the force measuring explained in Section \ref{sec:force_estimation}, showing the robustness of the algorithm.

In addition, we wanted to validate that the method can find the number of compliant axes which, together with the desired direction, can reproduce the demonstrated motion. In 2-D, this corresponds to choosing between 0 and 1 compliant axis. Fig. \ref{fig:comp_results} visualizes steps from the process of choosing the compliant axes. Fig. \ref{fig:unit} is in the same coordinate system as Fig. \ref{fig:comp_axes}, with the actual direction of motion from demonstrations 1 and 2 plotted in green crosses: in this simple~setup, the directions of motion between the demonstrations are close enough to each other that the two green crosses are overlapping but not near the origin, signalling the need for one compliant axis.

To properly validate the choice of 1 compliant axis, we computed the BIC values with all 2-demonstration pairs from the four available demonstrations. The results are in Fig. \ref{fig:dofplot}. Due to the similarity of demonstrations, the values are again overlapping heavily, but it can be seen that the difference between 0 and 1 compliant axes is clear. As required, the direction of this axis is perpendicular to $\pmb{\hat{v}}_{\rm d}^*$ in the 2-D plane. 

\begin{figure}[t]
\vspace{0.5cm}
\centering
\begin{subfigure}[b]{0.48\columnwidth} 
	\centering
	\includegraphics[width=0.70\columnwidth]{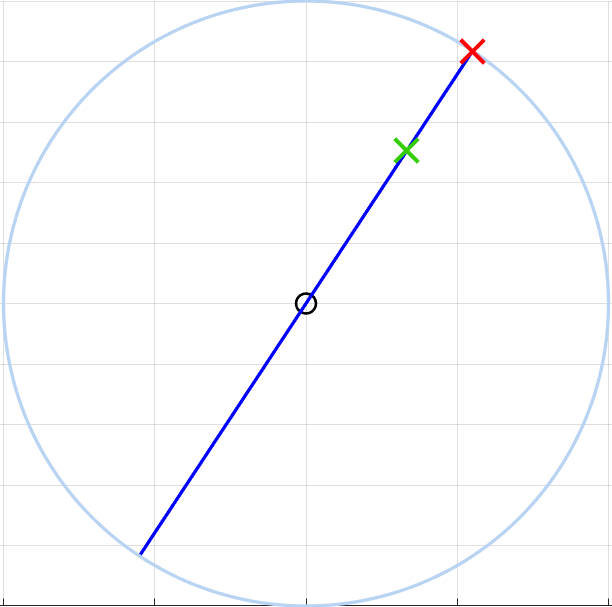}
	\caption{Illustration of the BIC computation}
	\label{fig:unit}
\end{subfigure}
\begin{subfigure}[b]{0.48\columnwidth} 
	\centering
	\includegraphics[width=0.85\columnwidth]{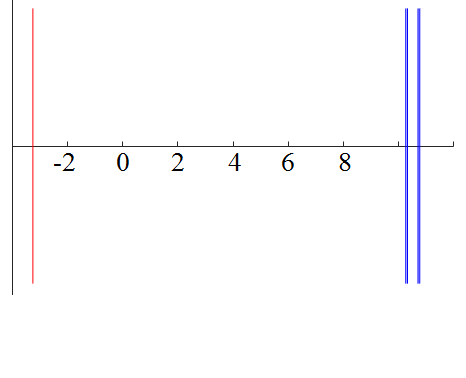}
	\caption{BIC values}
	\label{fig:dofplot}
\end{subfigure}
\vspace{0.1cm}
\caption{\small{(a) Average directions of motion of demonstrations 1 and 2 (green crosses overlapping each other) in the same coordinate system as in Fig. \ref{fig:comp_axes}: (b) BIC values of 0 (blue) and 1 (red) compliant axis based on likelihoods from Algorithm \ref{alg:compliant}. All 6 two-demonstration combinations of the 4 recorded demonstrations are plotted, and due to the similar values are overlapping heavily. The model with smaller BIC value is chosen.}}
\label{fig:comp_results}
\vspace{-0.2cm}
\end{figure}

Due to the inaccuracies in the manipulator's inverse kinematics, a small stiffness value has been put on the compliant axis (the stiffness value of $1/10$ on the axis corresponding~$\pmb{\hat{v}}_{\rm d}^*$). Consequently, $\mathbf{K}_{\rm d} = k_{\rm stiff}[0.75\ 0.40; 0.40\ 0.35]$ is obtained from the learning data using the proposed learning algorithm, and the value of $k_{\rm stiff}=4\times10^4 \frac{{\rm N}}{{\rm m}}$ is used in the experiments. Furthermore, $\mathbf{D}_{\rm d} = {\rm diag}(2.0, 2.4)\times10^3\frac{{\rm N}}{{\rm m/s}}$ is used for the desired damping (in line with the damping along the compliant axis used in \cite{Koivumaki_TMECH2017}). Finally, using the above $\mathbf{K}_{\rm d}$, $\mathbf{D}_{\rm d}$, and Condition~\ref{cond:gain2}, the learned control matrix $\pmb{\Lambda}_{\chi}$ and control matrix $\pmb{\Lambda}_{f}$ can be written as
\begin{equation}
	\pmb{\Lambda}_{\chi} = \begin{bmatrix} 15.00  & 8.06 \\ 6.72 & 5.83 \end{bmatrix}\frac{{\rm m/s}}{{\rm m}},\   \pmb{\Lambda}_{f} = \begin{bmatrix} 5.0  & 0 \\ 0 & 4.17 \end{bmatrix}10^{-3}\frac{{\rm m/s}}{{\rm N}}. \nonumber
\end{equation} 
Note that $\pmb{\Lambda}_{\chi}$ and $\pmb{\Lambda}_{f}$ qualify as positive-definite matrices, satisfying Condition \ref{cond:gain2}. \textit{This verifies that the proposed~learning algorithm provides a theoretically sound method~to~compute the impedance control matrix} $\pmb{\Lambda}_{\chi}$ \textit{from the learned} $\mathbf{K}_{\rm d}$.


\subsection{Reproduction of motion}
\label{sec:results_reprod}


To verify that the taught motions (in Fig. \ref{fig:learned_data}) can be reproduced with the learned parameters, \textit{two test cases} (\textit{TC}s) with different environment stiffness (shown in Fig. \ref{fig:setup}) were performed. \textit{In TC 1}, the environment was the same wooden pallet environment as was used in the learning (see Fig. \ref{fig:setup}a). \textit{In TC 2}, two wooden pallets were replaced with a pile of styrofoam sheets (see Fig. \ref{fig:setup}b) making the environment more compliant in relation to TC 1.  

Fig. \ref{fig:experiment_data} shows the manipulator behavior with the learned parameters in both TCs. In this figure, the data for the experiment with the wooden pallets (see Fig. \ref{fig:setup}a) is given in black and the data for the experiment with the wooden pallets and styrofoam sheets (see Fig. \ref{fig:setup}b) is given in grey.  

The first plot in Fig. \ref{fig:experiment_data} shows the Cartesian position~profile (in X-Y space) in both TCs. The starting position of the manipulator in both TCs is shown with a circle (and~can~be seen also in Fig. \ref{fig:setup}). As the plot demonstrates, in both TCs the taught motion can be reproduced with the learned parameters; see relation to Fig. \ref{fig:learned_data}. When contact to the environment was established, it was retained all the time during the motion in both TCs. Furthermore, in TC 2, the manipulator penetrates deeper along the Y-axis (in relation~to TC 1) due to more compliant environment.~This~was anticipated.

The second and third plots in Fig. \ref{fig:experiment_data} show the (estimated) Cartesian forces along the X- and Y-axes, respectively, in both TCs. As the plots show, almost identical force behaviors were identified in both TCs despite different compliance in the test environments. Furthermore, the force levels along the X- and Y-axes in Fig. \ref{fig:experiment_data} correspond well in relation to the average force levels during the teaching (see Fig. \ref{fig:learned_data}). 

In view of the results in Fig. \ref{fig:experiment_data}, \textit{it can be concluded that the proposed method provides efficient tools for LfD with heavy-duty hydraulic manipulators, despite all challenges discussed in Sections \ref{intro} and \ref{METHOD}.}            

\begin{figure}[t]
			\vspace{0.2cm}
      \centering
      \makebox{\parbox{\columnwidth}{\centering \includegraphics[width=\columnwidth]{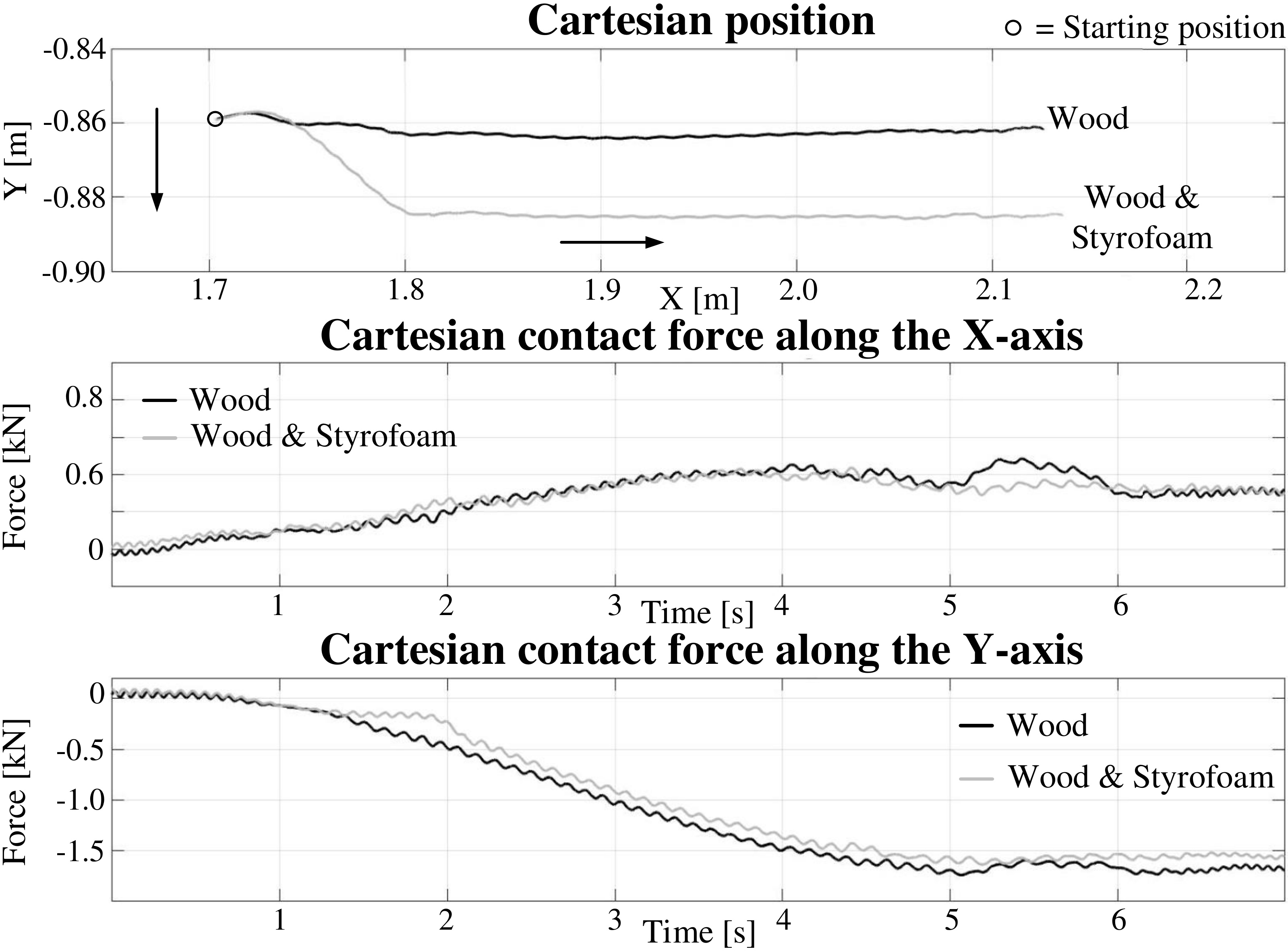}
      \vspace{-0.6cm}
			\caption{\small{The experimental results with the proposed learning.}}
      \label{fig:experiment_data}}}
			\vspace{-0.15cm}
\end{figure}


\section{CONCLUSIONS AND FUTURE WORK}
\label{CONCLUSION}
We presented, to the authors' knowledge, the first results on how LfD for in-contact tasks can be performed on a hydraulic manipulator. We showed how to overcome three significant obstacles restricting LfD implementation to hydraulics, namely how 1) to perform force-reflected bilateral teleoperation on a hydraulic manipulator while estimating the contact forces between tool and environment, 2) to learn from teleoperated demonstrations, which are noisier than demonstrations gathered with kinesthetic teaching, and 3) to reproduce the learned motions with a stability-guaranteed impedance controller for a hydraulic manipulator. 

Even though the test setup is simple, it demonstrates well the strength of the work. With the force-reflected bilateral teleoperation, even a non-expert can perform demonstrations without breaking the manipulator or the environment. Demonstrations of different lengths, which cannot be avoided on many real-world hydraulic applications, are not an issue due to the learning method combining them through intersections. Our controller can keep the applied force similar even though the characteristics of the setup vary, making the use of such heavy-duty machines safer. The downside of the learning algorithm in this paper is that it cannot learn nonlinear free-space motions; however, often in heavy-duty manipulation tasks such motions are not required. 

This paper adjusts and combines existing methods to demonstrate that hydraulic manipulators can be efficiently taught to autonomously perform motions requiring contact with the environment in the presence of positional uncertainty and varying environment conditions. Hydraulic heavy-duty machines are common in tasks which require heavy lifting, but the current LfD-research is concentrating mainly on electric manipulators and physically light tasks. On the other hand, research regarding hydraulics is often concentrating on the traditional control problems and not considering whether learning could be used to further facilitate the usage of hydraulic heavy-duty machines. We believe that this paper is an important milestone in bridging the gap between these two faculties and widening the view of both hydraulics and learning communities, as well as showing hydraulic industry that learning methods can be applied to heavy-duty machines.

Our future work consists of generalizing the method to a 6-DOF hydraulic manipulator and possibly another kind of hydraulic machine, such as an excavator. In addition, we plan to adapt the method from \cite{hagos2018} to work with hydraulic manipulators to perform whole tasks required in, for example, earthmoving or heavy maintenance.

\bibliographystyle{ieeetr}
\bibliography{biblio}

\appendices

\section{The Proof for Theorem \ref{thm:impedance}}
\label{proof:impedance}
Substituting \eqref{EQ_imp2} into \eqref{EQ_chi_r} and using $\pmb{\Lambda}_{f} = \mathbf{D}^{-1}_{\rm d}$ and $\pmb{\Lambda}_{\chi} = \mathbf{D}^{-1}_{\rm d}\mathbf{K}_{\rm d}$ in Condition~\ref{cond:gain2}, it yields
\begin{align}
	\dot{\pmb{\chi}}_{\rm r} &= \dot{\pmb{\chi}}_{\rm d} + \pmb{\Lambda}_{\chi}(\pmb{\chi}_{\rm d} - \pmb{\chi}) - \pmb{\Lambda}_f[\mathbf{D}_{\rm d}(\dot{\pmb{\chi}}_{\rm d} - \dot{\pmb{\chi}}) + \mathbf{K}_{\rm d}({\pmb{\chi}}_{\rm d} - {\pmb{\chi}})] \nonumber \\
														&= \dot{\pmb{\chi}}_{\rm d} + \mathbf{D}^{-1}_{\rm d}\mathbf{K}_{\rm d}(\pmb{\chi}_{\rm d} - \pmb{\chi}) - \mathbf{D}^{-1}_{\rm d}\mathbf{D}_{\rm d}(\dot{\pmb{\chi}}_{\rm d} - \dot{\pmb{\chi}}) \nonumber \\
														&\hspace{0.4cm}- \mathbf{D}^{-1}_{\rm d}\mathbf{K}_{\rm d}({\pmb{\chi}}_{\rm d} - {\pmb{\chi}}) \nonumber \\
														&= \dot{\pmb{\chi}}. \label{EQ_App3}
\end{align}

Then, using \eqref{EQ_chi_r}, \eqref{EQ_App3} and Condition~\ref{cond:gain2} yields
\begin{flalign}
  &\hspace{0.5cm}\dot{\pmb{\chi}}_{\rm r} = \dot{\pmb{\chi}}_{\rm d} + \pmb{\Lambda}_{\chi}(\pmb{\chi}_{\rm d} - \pmb{\chi}) + \pmb{\Lambda}_f({}^{\mathbf G}{\pmb{f}}_{\rm d} - {}^{\mathbf G}{\pmb{f}})& \nonumber\\
	&\hspace{0.5cm}\Leftrightarrow {}^{\mathbf G}{\pmb{f}}_{\rm d} - {}^{\mathbf G}{\pmb{f}} = -\pmb{\Lambda}^{-1}_f(\dot{\pmb{\chi}}_{\rm d} - \dot{\pmb{\chi}}_{\rm r}) - \pmb{\Lambda}^{-1}_f\pmb{\Lambda}_{\chi}(\pmb{\chi}_{\rm d} - \pmb{\chi})& \nonumber\\
	&\hspace{0.5cm}\Leftrightarrow {}^{\mathbf G}{\pmb{f}}_{\rm d} - {}^{\mathbf G}{\pmb{f}}= -\mathbf{D}_{\rm d}(\dot{\pmb{\chi}}_{\rm d} - \dot{\pmb{\chi}}) - \mathbf{D}_{\rm d}\mathbf{D}^{-1}_{\rm d}\mathbf{K}_{\rm d}(\pmb{\chi}_{\rm d} - \pmb{\chi})& \nonumber\\
	&\hspace{0.5cm}\Leftrightarrow {}^{\mathbf G}{\pmb{f}}_{\rm d} - {}^{\mathbf G}{\pmb{f}}= -\mathbf{D}_{\rm d}(\dot{\pmb{\chi}}_{\rm d} - \dot{\pmb{\chi}}) - \mathbf{K}_{\rm d}(\pmb{\chi}_{\rm d} - \pmb{\chi}).& \label{EQ_App4}
\end{flalign}

In \eqref{EQ_App4}, the first row equals to \eqref{EQ_chi_r}, whereas the last row equals to \eqref{EQ_imp2}. This completes the proof for Theorem~\ref{thm:impedance}. \hfill $\blacksquare$

\end{document}